%% file: paper.tex
\documentclass[letterpaper, 10 pt, conference]{ieeeconf}  

\IEEEoverridecommandlockouts                              
\usepackage{subcaption}
\usepackage{graphicx}

\title{\LARGE \bf
A Task-Motion Planning Framework Using Iteratively Deepened AND/OR Graph Networks
}

\author{Hossein Karami, Antony Thomas${^1}$ and Fulvio Mastrogiovanni
\thanks{All the authors are with the Department of Informatics, Bioengineering,
Robotics, and Systems Engineering, University of Genoa, Via Opera Pia 13,
16145, Genoa, Italy.}%

\thanks{${^1}$Corresponding author’s email: antony.thomas@dibris.unige.it.}%
}
\input{packages}
\input{commands} 

\begin{document}

\maketitle
\thispagestyle{empty}
\pagestyle{empty}

\begin{abstract}
We present an approach for Task-Motion Planning (TMP) using Iterative Deepened AND/OR Graph Networks (\textsc{TMP-IDAN}) that uses an AND/OR graph network based novel abstraction for compactly representing the task-level states and actions. While retrieving a target object from clutter, the number of object re-arrangements required to grasp the target is not known ahead of time. To address this challenge, in contrast to traditional AND/OR graph-based planners, we grow the AND/OR graph online until the target grasp is feasible and thereby obtain a network of AND/OR graphs. The AND/OR graph network allows faster computations than traditional task planners. We validate our approach and evaluate its capabilities using a Baxter robot and a state-of-the-art robotics simulator in several challenging non-trivial cluttered table-top scenarios. The experiments show that our approach is readily scalable to increasing number of objects and different degrees of clutter. 
\end{abstract}

\section{Introduction}
\label{sec:intro}
\input{intro}

\section{Task-Motion Formalism}
\label{sec:background}
\input{preliminaries}

\section{AND/OR Graph Networks}
\label{sec:AOgraph}
\input{graphnet}

\section{TMP-IDAN}
\label{sec:approach}
\input{approach.tex}

\section{Experimental Results}
\label{sec:results}
\input{results.tex}

\section{Conclusion}
\label{sec:conclusion}
\input{conclusion.tex}

\bibliographystyle{plain}
\bibliography{References}


\end{document}

%% file: packages.tex
\usepackage[english]{babel}
\usepackage[utf8]{inputenc}
\usepackage{epstopdf}
\usepackage{url}
\usepackage[bottom]{footmisc}

\usepackage{epstopdf}
\usepackage{makeidx}         
\usepackage{graphicx}        
\usepackage{multicol}        

\usepackage{amsmath,bm,amsfonts,amssymb}

\usepackage{commath}      
\usepackage{subfloat}
\usepackage{color,xcolor,ucs}
\usepackage[font=small,labelfont=bf]{caption}

\usepackage{floatrow}
\usepackage{tabularx}
\usepackage{float}
\usepackage{amsthm}

\usepackage{cite}
\usepackage{xr-hyper}
\usepackage{wrapfig}
\usepackage[ruled,vlined,linesnumbered]{algorithm2e}

\usepackage{multirow}
\usepackage{listings}
\usepackage{lipsum}
\usepackage{textcomp} 

%% file: commands.tex
 \SetKwInOut{Input}{require}
 \SetKwInOut{Output}{output}

\newtheorem{defn}{Definition}

\newtheorem{lemma}{Lemma}
\newtheorem{proposition}{Proposition}
\newtheorem{remark}{Remark}


%% file: intro.tex
Autonomous robots operating in the real-world are often faced with different levels of reasoning while executing their tasks. For example, while manipulating objects in clutter, a high-level reasoning is required to make discrete decisions regarding which object to grasp or push. To execute these decisions a robot has to reason in the continuous collision-free motion space. Planning in these discrete-continuous spaces is referred to as \textit{Task-Motion Planning} (TMP)~\cite{lagriffoul2018RAL}.

In the recent past, TMP has received considerable interest among the research community~\cite{kaelbling2013IJRR, srivastava2014ICRA,dantam2016RSS,garrett2018IJRR,thomas2019ISRR,garrett2019arxiv}. The primary challenge of planning in the hybrid space is to obtain an efficient mapping between the discrete task and continuous motion levels. The de facto standard syntax for task planning is the \textit{Planning Domain Definition Language} (PDDL)~\cite{mcdermott1998AIPS} and most approaches resort to the same for task planning. Semantic attachments are used in~\cite{dornhege2009SSRR,dornhege2009ICAPS,dantam2016RSS,thomas2019ISRR}, associating algorithms to functions and predicate symbols via external procedures. Though semantic attachments allow mapping between the task and motion space, the environment is assumed to be known before-hand. Besides, the robot configuration, grasp poses need to be specified in advance and renders the continuous motion space to be finite. This \textit{a priori} discretization is relaxed in~\cite{kaelbling2011ICRA,srivastava2014ICRA,toussaint2015IJCAI,garrett2018IJRR}. Yet, these approaches are domain specific in nature. Moreover, most approaches (for example~\cite{erdem2011ICRA,srivastava2014ICRA,dantam2016RSS}) first compute a task-level plan and refines it until a feasible motion plan is found or until timeout. The approach in~\cite{thomas2019ISRR}, however, checks for the motion feasibly as each task-level action is expanded by the task planner. But this approach assumes a pre-discretized motion space and thereby a finite action set. Recent work of Caelan \textit{et al.}~\cite{garrett2020ICAPS} address this limitation by introducing \textit{streams} within PDDL, which enable procedures for sampling values of continuous variables and thereby encoding an infinite set of actions. 

\begin{wrapfigure}{L}{0.15\textwidth}
\begin{center}
   \includegraphics[trim=1 9 2 9,width=1\textwidth]{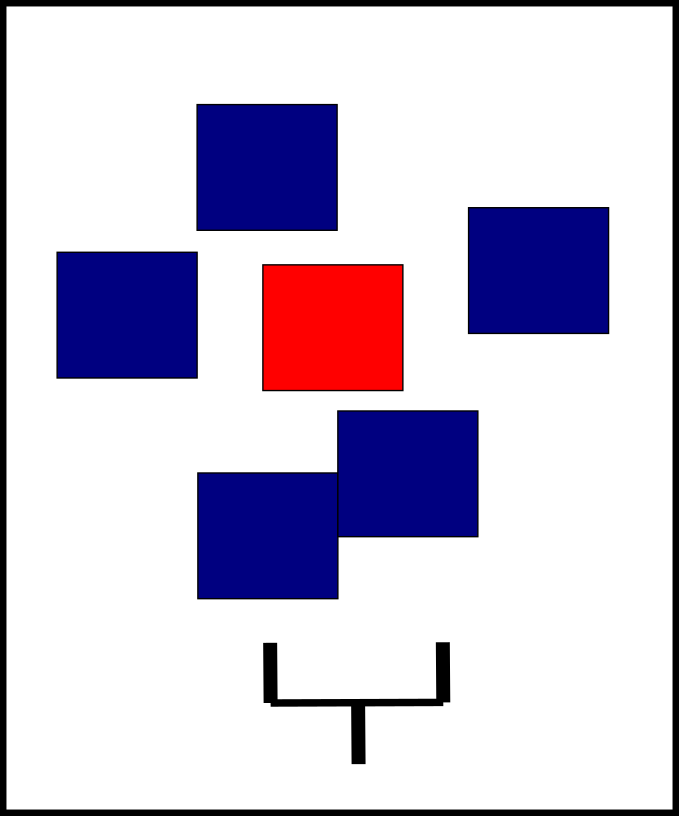}
  \end{center}
  \caption{Toy example in 2D; gripper needs to retrieve the red object from clutter.}
 \label{fig:toy}
\end{wrapfigure} 
Though off-the-shelf PDDL planners are available, one needs additional expertise to incorporate semantic attachments or streams that comply with the state-space search of the planner. Let us consider a simple 2D pick and place scenario\footnote{Also referred as cluttered table-top scenario in the paper.} shown in Fig.~\ref{fig:toy}, where a gripper needs to retrieve a target object (red) from clutter. As seen in the figure the target object is not immediately graspable and other objects (blue) that hinder the target grasp need to be re-arranged. This toy scenario may be modeled using two PDDL actions \texttt{pick} and \texttt{place}. Note that for a TMP problem each discrete action should also realize appropriate motions to achieve the desired tasks. For the PDDL domain file one may define the following predicates: \vspace{0.1cm} \\ 
{ \textsf{\hspace{0.5cm}   (\textbf{:predicates} (clear ?x)
              (gripper-empty)
             (holding ?x))}}
\vspace{0.1cm} \\
To keep the scenario simple enough, we deliberately make the assumption that if the target grasp is hindered, the robot is able to pick up another object/objects to make the target \texttt{pick} action feasible. We note that PDDL planners initiate the search through a process called grounding. Grounding is used to replace the variables in the predicates with different objects so as to instantiate predicates and action schemas. The predicates \textsf{clear\ ?x}, \textsf{gripper-empty}, and \textsf{holding\ ?x} checks if the object is graspable, gripper arm is empty, and if the gripper is holding the object. For 6 objects this gives rise to 13 propositions and therefore $n = 2^{13}$ states\footnote{Six possibilities for each of the predicates \textsf{clear\ ?x}, \textsf{holding\ ?x} and one for \textsf{gripper-empty}}. Shortest-path found using a state-transition graph via the \texttt{pick} and \texttt{place} actions would take $O(n \log n)$ time~\cite{bryce2007AIM}. Clearly, heuristic search avoids visiting large areas of the transition graph via informed search to arrive at the goal faster. Furthermore, the number of objects that need to be re-arranged is not known before-hand and this would require observation based re-planning within the PDDL framework. This requires additional mapping between the PDDL action space and the observation space of the robot. Moreover, the time taken for each re-planning further exacerbates the time complexity issue. Furthermore, PDDL planning is EXPSPACE-complete and can be restricted to less compact propositional encodings to achieve PSPACE-completeness~\cite{helmert2006PHD}. 

We present a probabilistically complete approach for TMP, specifically focusing on re-arranging a cluttered table-top to retrieve a target object. We tackle the two challenges described above, that is, (1) the number of objects to be re-arranged may not be known in advance and (2) the computational complexity of typical PDDL based planners. To address these challenges, we encode the task-level abstractions of the TMP problem efficiently and compactly within an AND/OR graph that grows iteratively until a solution is found. However, in general an AND/OR graph needs to be constructed offline and therefore it is required to know before-hand the number of objects to be re-arranged. Since this is unknown and depends on the degree of clutter, we introduce AND/OR graph networks that iteratively deepens during run-time till the target is retrieved. In Section~\ref{sec:AOgraph} we show that task-planning complexity is almost linear with respect to the number of graph iterations and validate the same in Section~\ref{sec:results}. 

The rest of the paper is organized as follows. Task-motion formalism are introduced in Section~\ref{sec:background}. We introduce AND/OR graph networks in Section~\ref{sec:AOgraph} and formalize our approach in Section~\ref{sec:approach}. The results are presented in Section~\ref{sec:results} and Section~\ref{sec:conclusion} concludes the paper.


%% file: preliminaries.tex
Task planning or classical planning is the process of finding a discrete sequence of actions from the current state to a desired goal state~\cite{ghallab2016book}.
\begin{defn}A \textit{task} domain $\Omega$ can be represented as a state transition system and is a tuple $\Omega = \langle S, A, \gamma, s_0, S_g \rangle$ where:
\label{def:one}
\end{defn}
\vspace{-0.7cm}
\begin{itemize}
\item $S$ is a finite set of states;
\item $A$ is a finite set of actions;
\item $\gamma : S \times A \rightarrow S$ such that $s' = \gamma(s, a)$;
\item $s_0 \in S$ is the start state;
\item $S_g \subseteq S$ is the set of goal states.
\end{itemize}
\begin{defn} The task \textit{plan} for a task domain $\Omega$ is the sequence of actions $a_0,\ldots,a_m$ such that $s_{i+1} = \gamma(s_i, a_i)$, for $i = 0,\ldots,m$ and $s_{m+1}$ \textit{satisfies} $S_g$.
\end{defn}
Motion planning finds a sequence of collision free configurations from a given start configuration to a desired goal~\cite{latombe1991robot}.
\begin{defn}A \textit{motion planning domain} is a tuple $M = \langle C, f, q_0, G \rangle$ where:
\end{defn}
\begin{itemize}
\item $C$ is the configuration space;
\item $f =\{0,1\}$, for collision ($f=0$) else ($f =1$);
\item $q_0 \in C$ is the initial configuration;
\item $G \in C$ is the set of goal configurations.
\end{itemize}
\begin{defn} A motion \textit{plan} for $M$ finds a collision free trajectory in $C$ from $q_0$ to $q_n \in G$ such that $f=1$ for $q_0,...,q_n$. Alternatively, A motion \textit{plan} for $M$ is a function of the form $\tau : [0, 1] \rightarrow C_{free}$ such that $\tau(0) = q_0$ and $\tau(1) \in G$, where $C_{free} \subset C$ is the configurations where the robot does
not collide with other objects or itself.
\end{defn}
TMP combines discrete task planning and continuous motion planning to facilitate efficient interaction between the two domains. Below we define the TMP problem formally.
\begin{defn}A \textit{task-motion planning} with task domain $\Omega$ and motion planning domain $M$ is a tuple $\Psi =\langle C, \Omega, \phi, \xi, q_0 \rangle$ where:
\end{defn}
\vspace{-0.2cm}
\begin{itemize}
\item $\phi : S  \rightarrow 2^ C$, maps states to the configuration space; 
\item $\xi : A  \rightarrow 2^ C$, maps actions to motion plans.
\end{itemize}
\begin{defn}The \textit{TMP problem} for the TMP domain $\Psi$ is to find a sequence of discrete actions $a_0,...,a_n$ such that $s_{i+1} = \gamma(s_i, a_i)$, $s_{n+1} \in S_g$ and a corresponding sequence of motion plans $\tau_0,...,\tau_n$ such that for $i = 0,...,n$, it holds that (1) $\tau_i(0) \in \phi(s_i) \ \textrm{and} \ \tau_i(1)  \in \phi(s_{i+1})$, (2) $\tau_{i+1}(0) = \tau_i(1)$, and (3) $\tau_i \in \xi(a_i)$. 
\end{defn}

%% file: graphnet.tex
An AND/OR graph is a graph which represents a problem-solving process~\cite{chang1971AI}. Below we provide a brief overview of AND/OR graphs; a detailed exposition can be found in~\cite{karami2020ROMAN}.

\begin{defn}
An AND/OR graph $G$ is a directed graph represented by the tuple $G = \langle N,H \rangle$ where:
\begin{itemize}
\item $N$ is a set of nodes;
\item $H$ is a set of hyper-arcs.
\end{itemize}
\end{defn}
For a given AND/OR graph $G$, $H = \{h_1,\ldots,h_m\}$, where $h_i$ is a many-to-one mapping from a set of child nodes to a parent node. Let us revisit the toy example considered in Section~\ref{sec:intro} (Fig.~\ref{fig:toy}). An AND/OR graph for this scenario is visualized as a sub-graph in Fig.~\ref{fig:pick_place_graph} consisting of the nodes \textsf{graspable}, \textsf{gripper-empty}, \textsf{object picked}, \textsf{target placed} and \textsf{object placed}. Note that we currently ignore the \textsf{current configuration} node (green color) and the graph extending from the \textsf{object placed} node in the figure, details of which will be described later. Each node $n_i \in N$ of the graph $G$ represents a high-level state, for example, \textsf{object picked}. To achieve the state \textsf{object picked}, the gripper must be empty, that is \textsf{gripper-empty} and the object should \textsf{graspable}. Thus the nodes \textsf{graspable}, \textsf{gripper-empty} and \textsf{object picked} are akin to the PDDL predicates \textsf{clear\ ?x}, \textsf{gripper-empty}, and \textsf{holding\ ?x}, respectively. The hyper-arc induces a mapping from the child nodes \textsf{gripper-empty} and \textsf{graspable} to the parent node \textsf{object picked}. In that sense, a hyper-arc induces a logical \textsf{AND} relationship between the child nodes/states, that is, all the child states should be satisfied simultaneously to achieve the parent state. Similarly, a single parent node can be the codomain for different hyper-arcs $h_i$. These hyper-arcs are in logical \textit{OR} with the parent node. Nodes without any successors or children are called the \textit{terminal} nodes. The terminal nodes are either a success node, that is, \textsf{target placed} or a failure node, that is, \textsf{object placed}. 

Let us again consider the toy example in Fig.~\ref{fig:toy}. If the number of objects to be re-arranged are known, then an AND/OR graph can be constructed using the the above mentioned states and corresponding transitions. Yet, the number of object re-arrangements is scenario dependent and not known ahead of time. The AND/OR graph representation thus seems incompatible. However, we make the following observation--- the abstractions defined for clutter scenario, that is, the states (nodes of $G$) and actions (hyper-arcs of $G$) remain the same irrespective of the number of object re-arrangements. Thus we can envision a graph $G$ that expands online by repeating itself until the target is retrieved. Nevertheless, each iteration of $G$ corresponds to a new work-space configuration (object arrangement) and to encode this aspect we augment $G$ with a virtual node that represents the current work-space configuration. 
\begin{defn}
For an AND/OR graph $G=\langle N,H \rangle$, an augmented AND/OR graph $G^a$ is a directed graph represented by the tuple $G^a = \langle N^a,H^a \rangle$ where:
\begin{itemize}
\item $N^a = \{N,n^v\}$ with $n_v$ being the virtual node;
\item $H^a = \{H,H^v\}$ with $H^v=\{h^v_i\}_{1 \leq i \leq |H^v|}$.
\end{itemize}
\end{defn}
The virtual node $n^v$ is called the root node of the augmented graph $G^a$. Each virtual hyper-arc $h^v_i$ induces a mapping between the virtual node and a node of the graph $G$. In our cluttered table-top scenario, the most trivial case corresponds to successfully grasping the target without any object re-arrangements. However, most often due to clutter, a motion plan to the target do not exist and objects need to be re-arranged to obtain a feasible plan. In general, motion planner failure can arise if a path indeed does not exist or because the planning time allotted was insufficient. In this work, we assume that sufficient time is allotted to the planner so that a motion planning failure implies objects obstructing the path to the target. 

\begin{remark}
If the target grasp is unsuccessful, that is, the motion planner fails, then at least one object need to be re-arranged to search for a new target-graspable path.
\label{remark1}
\end{remark}
 Thus, following the root node, a success node is reached if the \textsf{target placed} task is a achieved and hence the graph is \textit{solved}. Else, the graph is terminated at the failure node (\textsf{object placed}) following the re-arrangement of a single object and the graph is therefore \textit{unsolved}. In our approach a single augmented AND/OR graph $G^a$ represents such a problem solving process. It readily follows from Remark~\ref{remark1} that if the graph $G^a$ is terminated at the failure node a re-attempt is to be made to achieve a grasp of the target object. Such an attempt can again lead to a \textit{solved} or \textit{unsolved} graph. Therefore to achieve the required objective, it is necessary to iterate $G^a$ till a success node is reached, that is, $G^a$ is solved. This \textit{roll-out} of augmented AND/OR graphs give rise to an AND/OR graph network.   
\begin{defn}
An AND/OR graph network $\Gamma$ is a directed graph $\Gamma = \langle \mathcal{G},T \rangle $ where:
\begin{itemize}
\item $\mathcal{G} = \{G^a_1,\ldots,G^a_{n'}\}$ is a set of augmented AND/OR graphs $G^a_i$;
\item $T= \{t_1,\ldots,t_{n'-1}\}$ is a set of transitions such that $G^a_{i+1} = t_i(G^a_i), \ 1 \leq i \leq n'-1$.
\end{itemize}
\end{defn}

\begin{figure}[]
\includegraphics[width=0.75\textwidth]{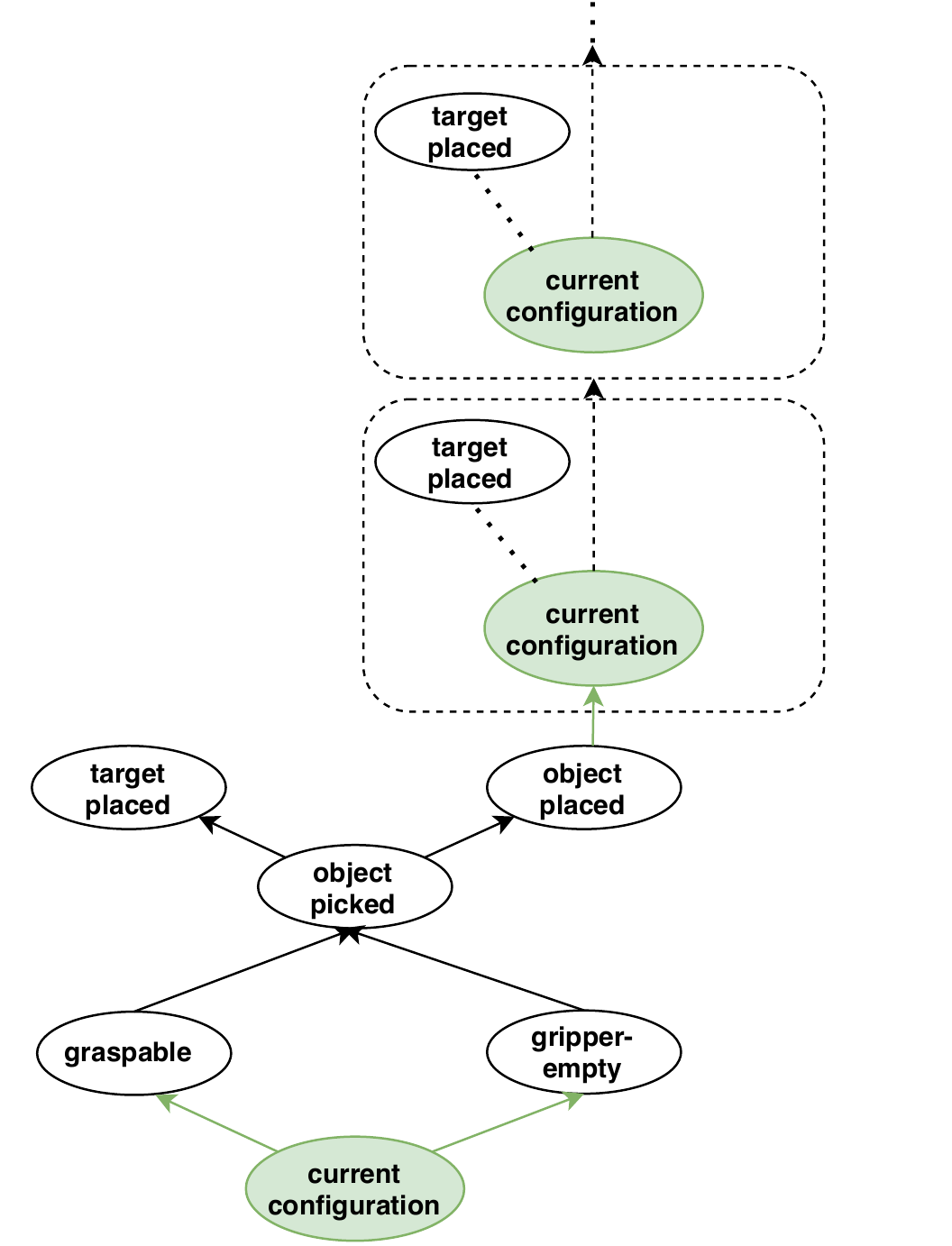}
\caption{AND/OR graph network for the pick and place scenario.}
\label{fig:pick_place_graph}
\end{figure}

\noindent where $n'$ is the total number of graphs in the network. Alternatively, $n'$ is also the depth of the network. Note that $t_i$ is defined only if $G^a_i$ is unsolved, that is, if the graph terminates at the failure node. This transitions $G_i$ to a new augmented AND/OR graph $G^a_{i+1}=t_i(G^a_i)$, updating the root node of $G^a_{i+1}$ to the changed work-space configuration. Thus for the pick and place scenario, we obtain an AND/OR graph network as shown in Fig.~\ref{fig:pick_place_graph}. The root node is the virtual node that represents the current work-space configuration (green node, that is, \textsf{current configuration}). The object to be picked and placed is decided by a cost function, for example, the proximity to the gripper and therefore for a single pick operation we need to traverse only the nodes till \textsf{object picked}. If the target object is being held, it is then put-down and the graph terminates, else a new object is selected to clear the path to the target. Once the object is moved, a new graph is grown with the \textsf{current configuration} being the updated current work-space configuration taking into account the object movement.

\noindent \textit{Complexity comparison:} In Section~\ref{sec:intro} we have discussed the computational complexity for the toy example in Fig.~\ref{fig:toy} modeled using PDDL. We now analyze the complexity aspects when the same example is modeled using our AND/OR graph network (see Fig.~\ref{fig:pick_place_graph}). Assuming sufficient time is allotted for the motion planner, in the worst case all the $5$ objects need to be re-arranged leading to $5\times5$ states\footnote{5 nodes for each object as seen in Fig.~\ref{fig:pick_place_graph} and 5 iterations due to 5 objects.} as opposed to $2^{13}$ using PDDL (see Section~\ref{sec:intro}). However it can happen that the task execution fails, for example due to motion, actuation or grasping errors. In such a case a new graph $G^a_{i+1}$ is grown retaining the same work-space configuration as $G^a_{i}$. Therefore the number of iterations $m$ may be greater than the number of objects (here $m>5$) and hence the time complexity is thus $O(5m) \approx O(m)$. Note that for PDDL based planning, shortest path takes $O(n \log n)$ time where $n = 2^{13}$ is the number of states. In general for an AND/OR graph network with each graph consisting of $n$ nodes, the time complexity is only $O(nm)$. Similarly, for an AND/OR graph with $n$ nodes a storage of only $O(n)$ nodes is required.

\begin{proposition}
An AND/OR graph network $\Gamma = (\mathcal{G},T)$ is said to be solved at depth $d$ when $G^a_d$ is solved.
\label{prop}
\end{proposition}
Following Proposition~\ref{prop}, an AND/OR graph net with underlying augmented AND/OR graphs $G^a_i$ is expanded till a $G^a_i$ is solved. Note that as discussed in Section~\ref{sec:intro}, we are interested in problems with solutions. For example in a cluttered table-top TMP scenario where a target object is to be grasped, we assume that a feasible grasp exist. This may result in a network of infinite depth. However, as a consequence of our assumption, that is, a solution exist, the network is solved at infinite depth. Moreover, to alleviate the issue of infinite depth, a predetermined depth limit may be employed. Thus, as the predetermined limit approaches infinity a solution is found. Practically, solutions with larger depth limits may be discarded as it may be too time consuming to arrive at. This enables us to define some sort of probabilistic completeness as will be seen later in Section~\ref{sec:approach}.  



%% file: approach.tex
We present a novel TMP approach \textsc{TMP-IDAN} (Task-Motion Planning using Iterative Deepened AND/OR Graph Networks) that uses AND/OR graph networks to compactly encode the task-level abstractions to reduce the overall task planning complexity. 

\subsection{System's Architecture}
\label{sub:architecture}
An overview of \textsc{TMP-IDAN} is given in Fig.~\ref{fig:arch}. In general the approach requires (1) \textit{perception capabilities} to comprehend objects in the scene, (2) \textit{task planning} to select the abstract actions and finally (3) \textit{motion planning} for executing the action to manipulate objects, avoiding potential collisions. The perception capabilities are encapsulated in a single module, which is called the \textit{Scene Perception}. This module provides the \textit{Knowledge Base} module with the information about the current work-space configuration, that is, the location of objects and the configuration of robot. The planning layer is made up of two modules, namely the \textit{TMP Interface}, and the \textit{Motion Planner} modules. The \textit{TMP Interface} module receives discrete or symbolic commands from the \textit{Task Planner} and maps them to actual geometric values and drives the behavior of the \textit{Motion Planner}. This module retrieves information regarding the work-space and robot from the \textit{Knowledge Base}. It also provides an acknowledgment to the \textit{Task Planner} upon the execution of a command by the robot. The \textit{Motion Planner} module plans the outcome of robot behaviors before their actual execution. 

\begin{figure}[]
\includegraphics[width=0.75\textwidth]{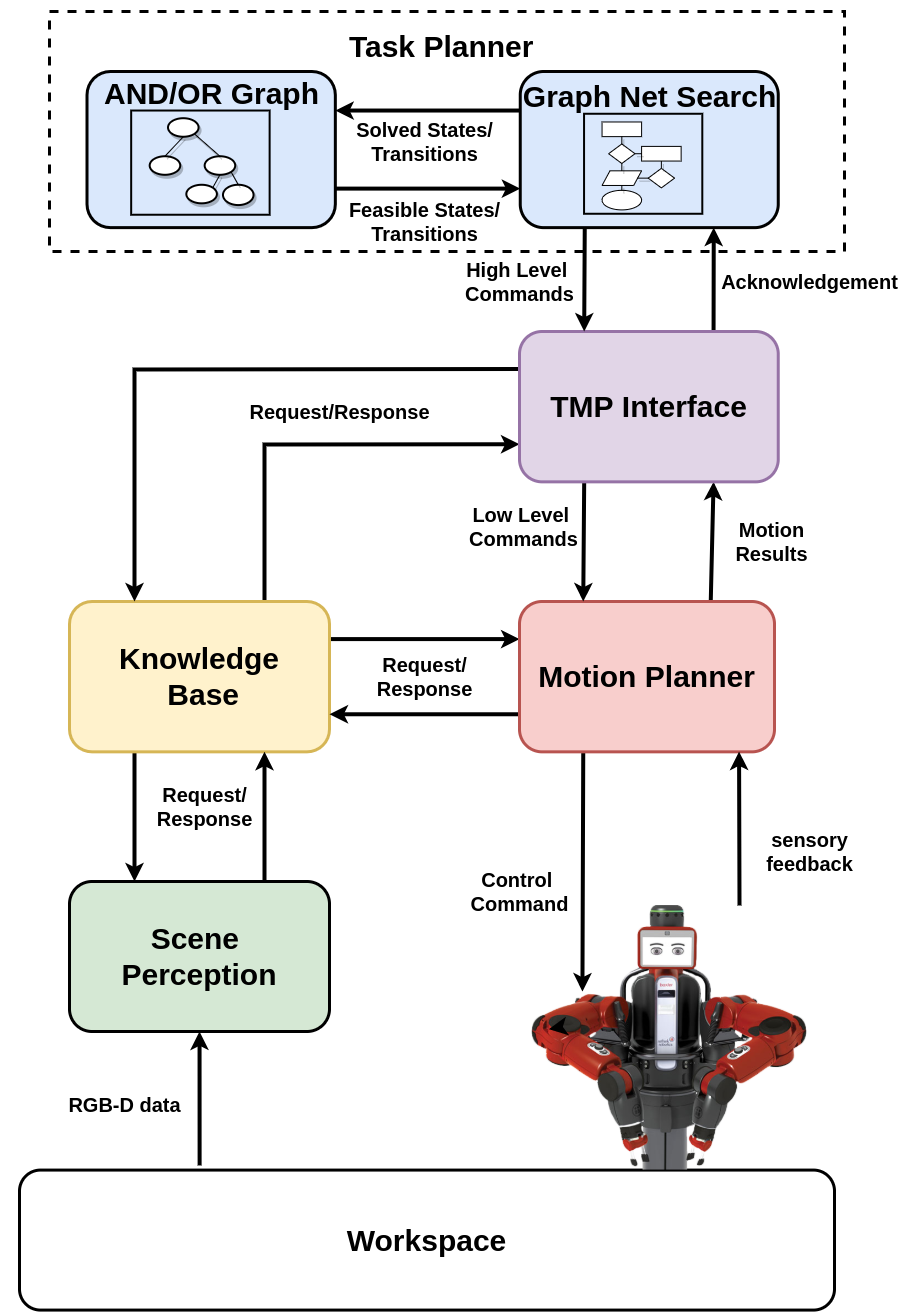}
\caption{System's architecture of \textsc{TMP-IDAN} framework.}
\label{fig:arch}
\end{figure}

The \textit{Task Planner} module embeds the augmented \textit{AND/OR Graph} and the AND/OR \textit{Graph Net Search}. Along with the AND/OR graph, the \textit{Task Planner} module is in charge of decision making and adaptation of the ongoing parallel (by parallel we mean different actions feasible from the current state) hyper-arcs. To do so, the \textit{Task Planner} provides a set of achieved transitions between the states to the \textit{Graph Net Search} module and receives the set of allowed cooperation states and transitions with the associated costs to follow. It then associates each state or state transition with an ordered set of actions in accordance with the work-space and robot configuration and finally incorporates the online simulation results to assign actions to the robot arms. Once an action is carried out, it receives an acknowledgment from the planning level and updates its internal structure. The \textit{Knowledge Base} stores all relevant information to make the cooperation progress. 

\begin{algorithm}[t]
\SetAlgoLined
\Input{Augmneted AND/OR Graph (AO), Task Planner (TP), TMP Interface (TMPI), Motion Planner (MP), Knowledge Base (KB), Scene Perception (SP)}
 \While{Target not retrieved}{
 AO:AddNewGraph()\;\label{AO}
 \SetKwBlock{blocka}{\textnormal{ TP: Request AO NextFeasibleStates()\label{TP}}}{}
\blocka{
 \eIf{Request = empty}{
   Go to Line~\ref{AO}\;}{
   continue\;
  }
  Find NextOptimalState()\;\label{NOP}
  \For{\textnormal{tasks, agents} in OptimalState}{
   SendToTMPI(tasks,agents)\;
   \For{\textnormal{task} in tasks}{
   \For{\textnormal{agent} in agents}{
   TMPI: RequestKB()\;
   KB: RequestSP()\;
   TMPI: RequestMP()\;
   }
   }
   }
   }
   \eIf{RequestMP()}{
   \SetKwBlock{blockb}{\textnormal{TP: FindOptimalMotionPlan()}}{}
\blockb{
TMPI: T $\leftarrow$ Optimal Motion Plan\;
\eIf{T executed}{Go to Line~\ref{TP}\;}{Go to Line~\ref{AO}\;\label{retry}}
   }}{Go to Line~\ref{NOP}\;}
   }
 \caption{\textsc{TMP-IDAN}}
 \label{algo}
\end{algorithm}
\subsection{Task Planning with AND/OR Graph Networks}
The AND/OR graph representation provides a framework for the planning and scheduling of task sequences~\cite{sanderson1988TAES}. 
The feasibility of the tasks is then checked using a suitable motion planner~\cite{karami2020ROMAN}. Moreover, an AND/OR graph inherently requires fewer nodes than the corresponding complete state transition graph, reducing the search complexity of the AND/OR space~\cite{sanderson1988TAES}. Yet, such a representation requires that the number of object re-arrangement to retrieve a target from clutter is known ahead of time. This representation thus seems incompatible as we do not know before-hand the number of objects to be re-arranged. To address this challenge, we introduce AND/OR graph networks as discussed in Section~\ref{sec:AOgraph} wherein an augmented AND/OR graph grows online until the target grasp is achieved. For the cluttered table-top scenario, the AND/OR graph network with two graphs is seen in Fig.~\ref{fig:combined} (right). Given the initial work-space configuration, a virtual node representing the same (\textsf{INIT\#0}) is added giving the augmented AND/OR graph $G_0^a$. The graph $G_0^a$ encodes the fact that if a feasible grasping trajectory exists then the \textsf{picked target} task is to be performed and otherwise an object is to be identified to be either removed  or pushed. Let us consider the case that a target grasping trajectory does not exit and that a \textsf{grasped closest object of target} followed by \textsf{object placed in storage} is achieved which is a failure node and thus exhausting $G_0^a$. This leads to a new graph $G_1^a$, which corresponds to a new augmented graph, with the same states and actions as $G_0^a$ but a different work-space configuration encoded via the virtual node \textsf{INIT\#1}. This process iteratively repeats itself until a graph $G_{n'}^a$ is solved which represents an AND/OR graph network $\Gamma$ of depth $n'$. The augmented AND/OR graphs $G^a_0,\ldots, G^a_{n'}$ are thus iteratively deepened to obtain an AND/OR graph network whose depth $n'$ is task depended and not known before-hand. We note here that in case a task in $G_i^a$ fails, for example due to motion, actuation or grasping errors, $G_{i+1}^a$ (with current work-space configuration) is grown and in this way our approach is robust to execution failures. 

Algorithm~\ref{algo} describes the overall planning procedure. It proceeds with the an AND/OR graph augmented with the initial work-space configuration (line~\ref{AO}) which is initiated by the AND/OR graph module via the call to \textit{AddNewGraph}. The task planner module then checks for feasible states (call to \textit{NextFeasibleStates}) in the augmented graph (line~\ref{TP}). From among the feasible states the optimal state is selected with the call to \textit{NextOptimalState}. Our cost function is a combination of distance of the object to the robot base, distance to the left, right robot gripper and the size of the object. The tasks and agents (left and right robot arm) of the feasible state are communicated to the TMP interface with the call to \textit{SendToTMPI}. Upon call to \textit{RequestKB} and \textit{RequestSP}, the geometric location of the objects, grippers are communicated to the TMP interface though the knowledge base and scene perception modules. A motion plan is sought for by the TMP interface with the \textit{RequestMP}. If the motion plan execution fails, for example a grasping failure, then a re-try is attempted (line~\ref{retry}). If a motion plan is not found then another object is selected for re-arrangement (line~\ref{NOP}); see Remark~\ref{remark1}. If motion plan is successful, a new graph is expanded and the process repeats until the target is retrieved. 
\vspace{-0.1cm}
\subsection{Probabilistic Completeness}
We now prove the probabilistic completeness of TMP-IDAN.
\begin{lemma}
For a predetermined depth limit $l$, TMP-IDAN is probabilistically complete.
\end{lemma}
\begin{proof}
For motion planning, we use RRT motion planner~\cite{kuffner2000ICRA} which is probabilistically complete~\cite{karaman2011IJRR}. Thus, given sufficient time, the probability of finding a plan, if one exists, approaches one. For the AND/OR graph network $\Gamma = (\mathcal{G},T)$ based task planner, by design each graph $G^a_i$ terminates either at a success node or a failure node and hence each $G^a_i$ complete. Using a predetermined depth limit $l$, $\Gamma$ is expanded at most till depth $l$. Thus if $G^a_j$, $j\leq l$ is solved then $\Gamma$ is solved. Else $\Gamma$ is terminated at $G^a_l$ and a non-existence of solution is reported. In reality, a solution does not exist at all or the chosen $l$ is shallow. However, the first scenario (non-existence of a solution) is quashed since we are interested in TMP scenarios with solutions (consequence of our assumption). Thus, without any loss of generality, it can be argued that as $l$ approaches infinity the probability of $\Gamma$ being solved asymptotically approaches one. \end{proof}

\begin{figure}[t]
    \centering
    \includegraphics[width = 0.99\textwidth]{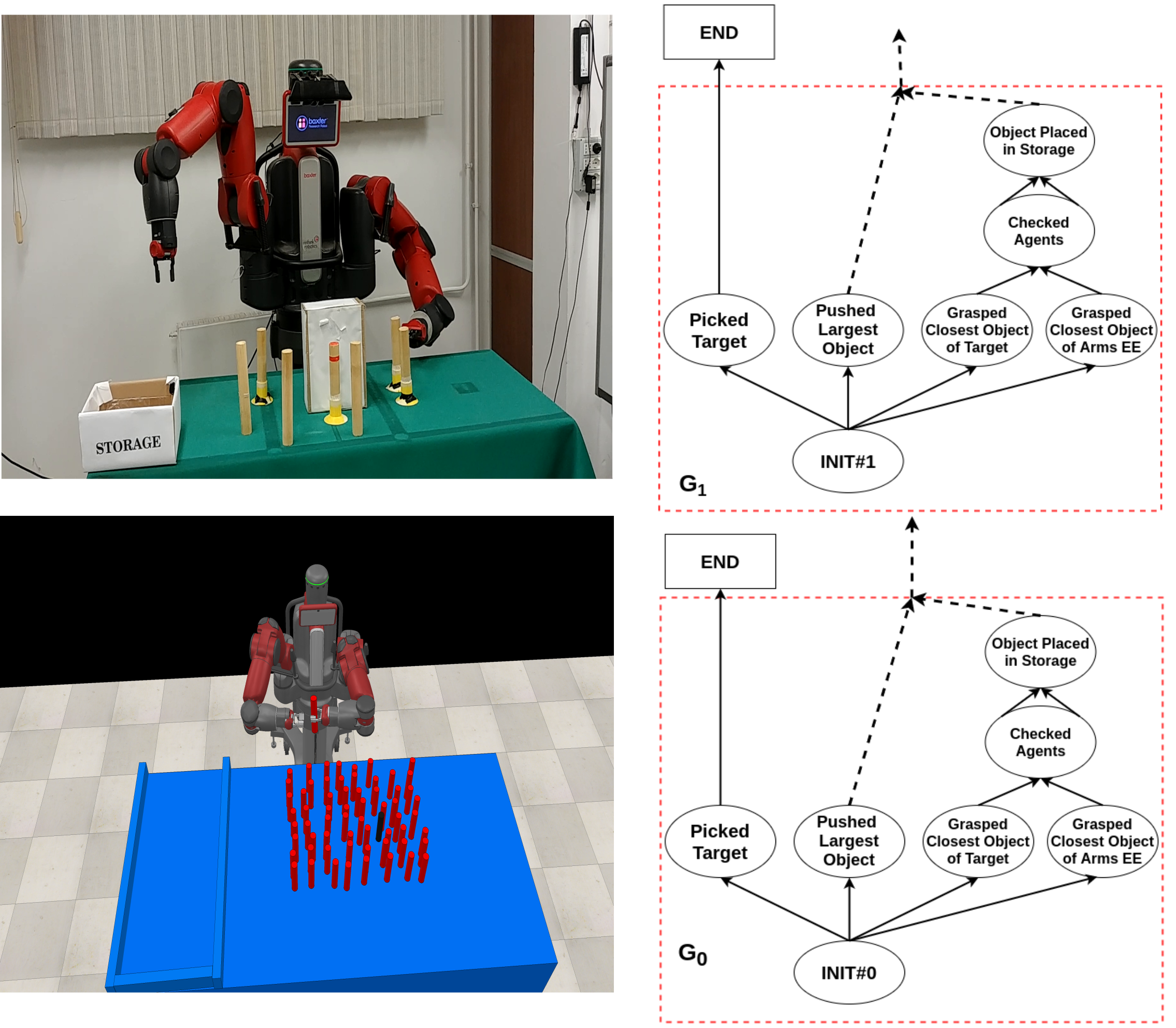}
    \vspace{-0.1cm}
    \caption{(top-left) \textsc{TMP-IDAN} in real world (bottom-left) and in simulation. (right) AND/OR graph net for the clutter scenario.}
    \label{fig:combined}
    \end{figure}

%% file: results.tex




\begin{figure*}[ht]
\centering
 \subfloat[]{\includegraphics[width=0.34\textwidth]{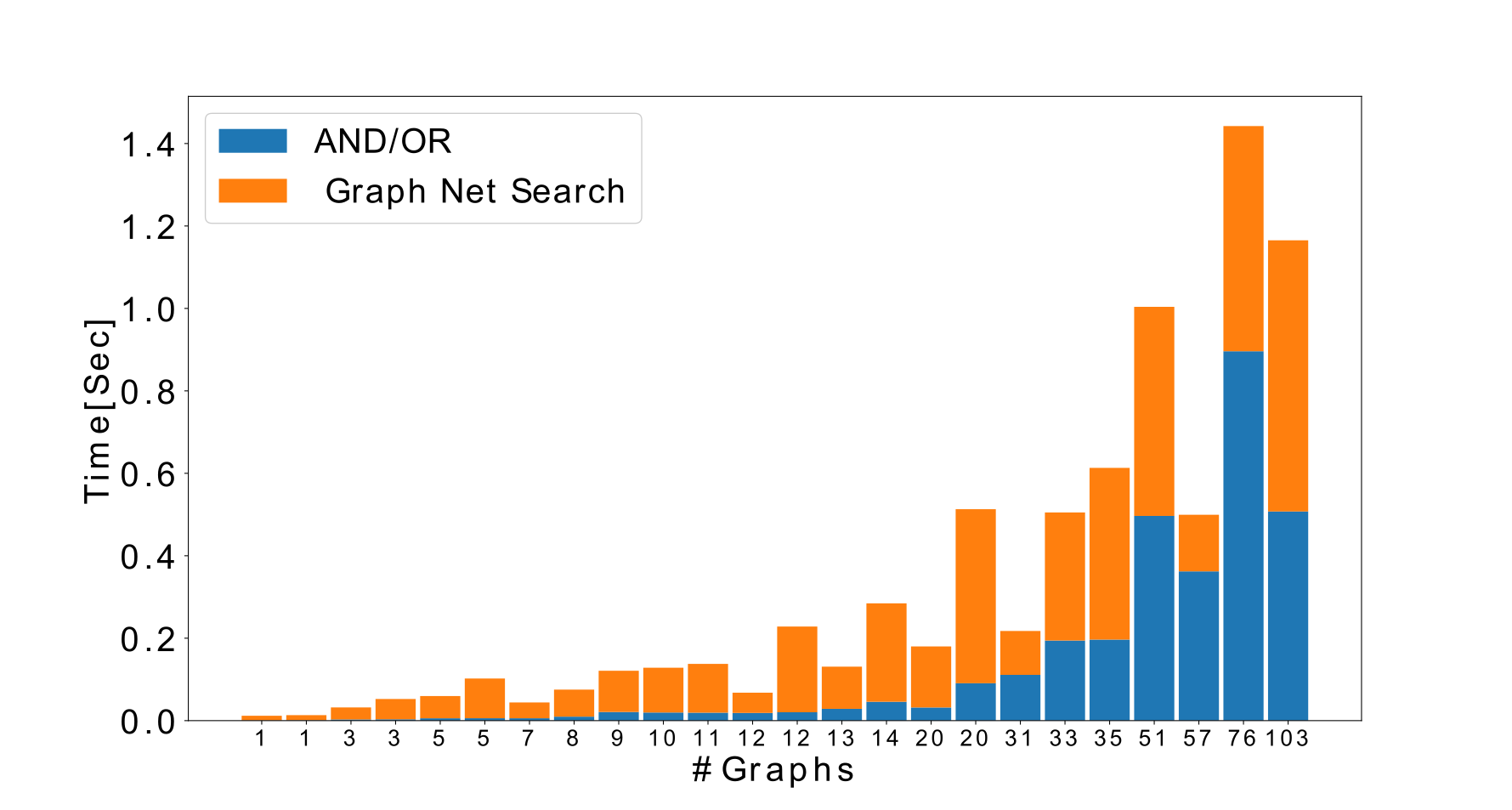}\label{fig:compa}}
  \subfloat[]{\includegraphics[width=0.34\textwidth]{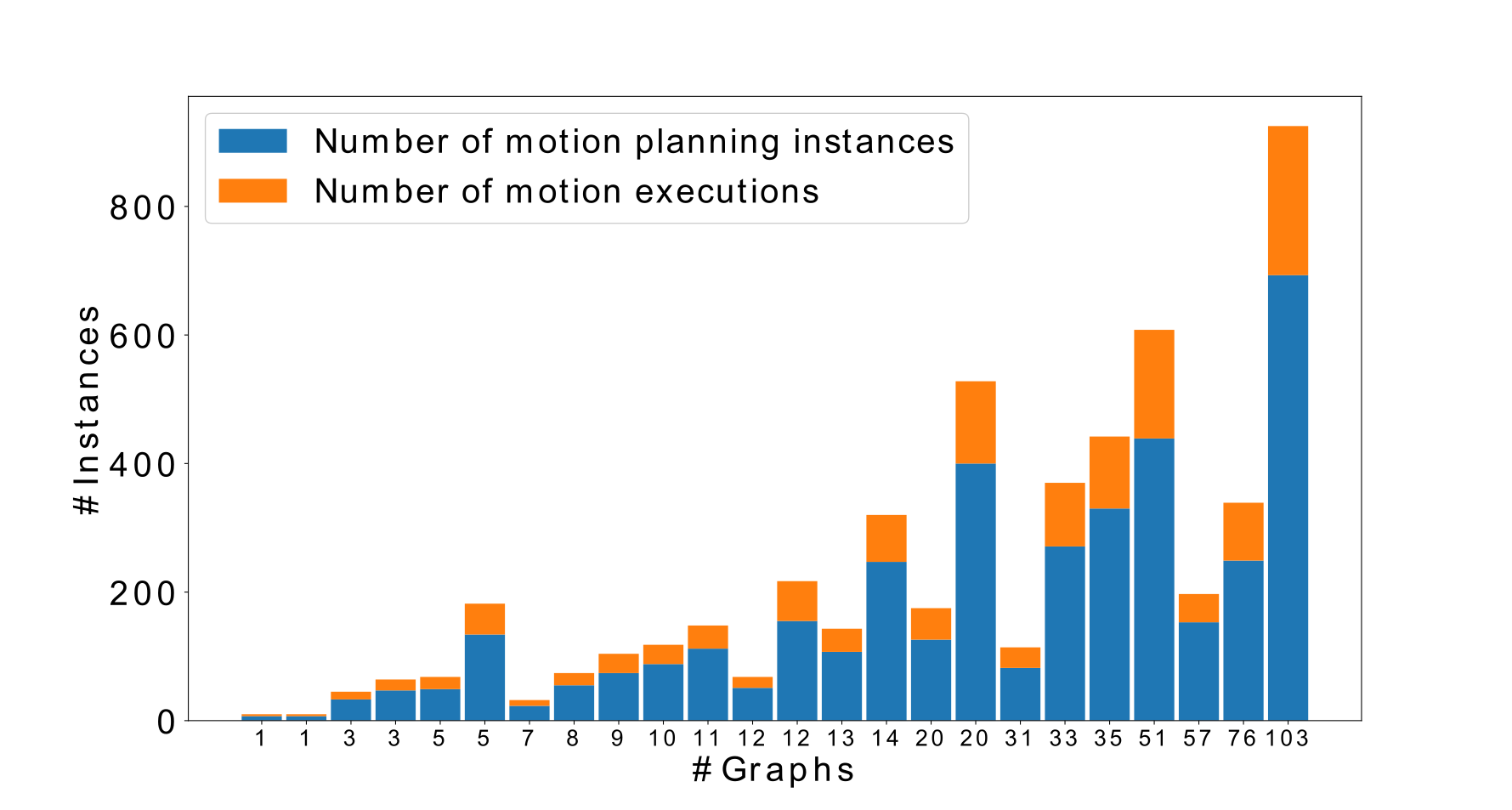}\label{fig:compb}}
  \subfloat[]{\includegraphics[width=0.34\textwidth]{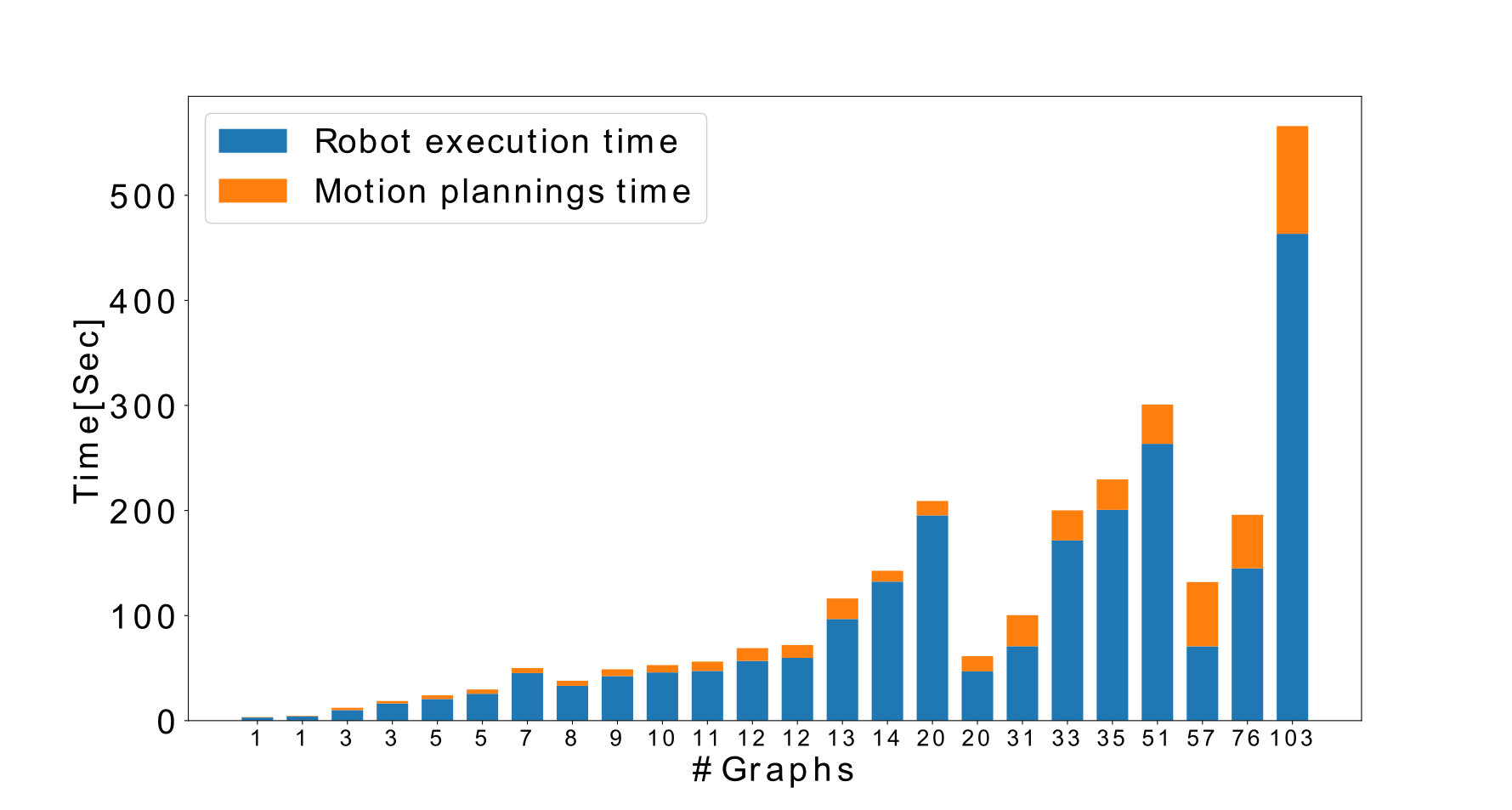}\label{fig:compc}}
  \vspace{-0.1cm}
\caption{Various histograms with increasing network depth.}
\label{fig:comp}
\end{figure*}
The experiments are performed on Rethink Robotics dual-arm Baxter robot which is equipped with standard grippers and an RGB-D camera mounted on its head and pointing downward which is used to acquire images for object detection. The benchmarks were conducted on a workstation equipped with an Intel(R) core i7-8700@3.2 GHz $\times$ 12 CPU's and 16 GB of RAM. The architecture is developed using C++ and Python under ROS Kinetic. In order to validate the effectiveness of \textsc{TMP-IDAN}, we consider a cluttered table-top scenario wherein a target object is to be retrieved from among clutter (see Fig.~\ref{fig:combined}(top-left and bottom-left)). The experiment scenario includes two physical agents, that is, the right and left arms of the Baxter and a table-top with cylinders and cuboids. Removed objects are placed in a storage close to the right arm and in case the object is picked by the left arm, it is handed to right arm and finally placed in the storage. This sequence goes on until the target grasp is feasible. A video demonstrating the results can be found at \url{https://youtu.be/QV5Y1iyicQo}.
\subsection{Description of the experiments}
To verify the adaptability of \textsc{TMP-IDAN} in real-world scenarios and it's robustness with respect to grasping failures, we first employ \textsc{TMP-IDAN} in real world with Baxter robot. Subsequently we perform twenty four different experiments using the state-of-the-art robotics simulator CoppeliaSim~\cite{coppeliaSim} to corroborate that our approach is highly fast and has a linear computational time profile with respect to number of expanded graphs. Moreover, simulation environment enables us to induce scenarios where number of objects is not known in advance and also vary the target object location randomly. In simulation, we start with four objects in the scene and increase the complexity by adding up to 64 objects. In all the runs, the position of target object is chosen randomly. Fig.~\ref{fig:combined} (right) shows the AND/OR graph network employed. For the initial graph $G_0$ there exists several parallel hyper-arcs of which picking the target object is given the minimum cost. If the target is not reachable by any of the agents, then closest object to target is set to be removed and once a feasible hyper-arc is found, then agents move forward in the graph with the set of assigned actions. 
\begin{table}[h]
\centering
\scalebox{0.8}{
\begin{tabular}{|l|c|c|} 
\hline
Module              &Average (s)  & Std. dev(s) \\ 
\hline
AND/OR graph & 0.0031          & 0.0028 \\ 
Graph network        & 0.0104          & 0.0048\\
Motion planning attempts      & 9.3980        & 5.6640\\
Motion executions      & 3.3540       & 1.9320\\
Motion planning time               & 0.8010      & 0.2170\\
Motion execution time   &4.7490    &2.0240\\
\hline
\end{tabular}}
\caption{Computation times for different modules of \textsc{TMP-IDAN}.}
\label{tab:compana}
\end{table}
\vspace{-0.6cm}
\begin{table}[h]
\centering
\scalebox{0.75}{
\begin{tabular}{|c|c|c|c|c|c|} 
\hline
Objects & d & TP (s) & MP (s) & MP attempts & Objects re-arranged\\
\hline
4 & 1.67 & 0.018 &1.044 & 15.66 & 1.33\\
\hline 
8 & 7.33 & 0.068 & 4.560 & 50.66 &4\\
\hline
15 & 14.33 & 0.170 & 20.893 & 177 & 7.5\\
\hline
20 & 57 & 0.422 & 61.168 & 400 & 18\\
\hline 
30 & 19.66 & 0.188 & 28.591 & 159.66 & 9\\
\hline
42 & 72 & 0.462 & 44.148 & 384.5 & 18.5\\
\hline
49 & 26.5 & 0.342 & 24.265 & 201 & 10\\
\hline
64& 76 & 0.657 & 102.575 & 693&29\\
\hline
\end{tabular}}
\caption{$d$- average network depth, TP- average total task planning time, MP- average total motion planning time, average motion planning attempts and the average number of objects to be re-arranged as the degree of clutter is increased.}
\label{tab:table1}
\end{table}

\subsection{Validation}
Table~\ref{tab:compana} shows the average planning and execution time for each module discussed in Section~\ref{sec:approach}. We increase the number of objects on the table and for each object number, we perform the experiment 3 times by randomly sampling the target location. Table~\ref{tab:table1} reports the average network depth $d$, average total task planning time, average total motion planning time, average number of motion planning attempts and the average number of objects to be re-arranged. Task planning times with increasing network depth corroborates our discussion on time complexity (see Section~\ref{sec:AOgraph}) and is almost linear with respect to $d$. For motion planning, we use Moveit package \cite{sucan2013moveit} as RRT planner within OMPL~\cite{sucan2012RAM} and we place a time bound of one second. However, motion planning failures due to actuation errors, grasping failures or occlusion lead to re-plan, explaining the large number of motion planning attempts. 

In addition to the computational complexity analysis in Table~\ref{tab:table1}, Fig.~\ref{fig:comp} shows different histogram plots with increasing network depth $d$. Fig.~\ref{fig:comp}(a) shows the total task planning time with $d$. As seen above, the planning times are almost linear with increasing $d$. However, slight deviations are readily observed. For example, the time for $d=76$ is greater than the time for $d=103$. This is due to the fact that in many cases, due to motion planning failure a new graph is expanded before reaching the terminal node. Thus for $d=76$, more number of nodes are traversed when compared to $d=103$ and rightly justifies the histogram. Fig.~\ref{fig:comp}(b) reports the number of motion planning attempts and the total executions with increasing $d$. Note that this depends on the degree of clutter and thereby the number of object re-arrangements required. This is more clearly observed in the last two columns of Table~\ref{tab:table1} where increase in motion planning attempts can be seen with increased object re-arrangements. Finally, in Fig.~\ref{fig:comp}(c), total execution times with $d$ can be seen.

%% file: conclusion.tex
We present a novel approach \textsc{TMP-IDAN}, a framework for task-motion planning that provides a compact abstraction for task-level representation using AND/OR graph networks. This reduces the overall time and space complexity of task planning. The network representation handles challenge of unknown object re-arrangements and allows online expansion of graphs till the target is retrieved. We use off-the-shelf motion planner and therefore our approach can readily be used with any motion planner. We validate \textsc{TMP-IDAN} on a physical Baxter robot and in simulation and it seen that our approach works with different degrees of clutter. Task-planning time scales linearly with respect to increasing network depth and this enables fast online computation in heavily cluttered non-trivial scenarios. Immediate future work includes reducing the depth of the graph by developing an efficient heuristic for object selection for re-arrangement and extension to human-robot and robot-robot collaboration. 